\documentclass{article}

\usepackage{microtype}
\usepackage{graphicx}
\usepackage{subfigure}
\usepackage{booktabs} % for professional tables

\usepackage{hyperref}

\usepackage[accepted]{icml2023}

\usepackage{amsmath}
\usepackage{amssymb}
\usepackage{mathtools}
\usepackage{amsthm}

\usepackage[capitalize,noabbrev]{cleveref}

\theoremstyle{plain}
\newtheorem{theorem}{Theorem}[section]

\newtheorem{lemma}[theorem]{Lemma}

\theoremstyle{definition}
\newtheorem{definition}[theorem]{Definition}
\newtheorem{assumption}[theorem]{Assumption}

\theoremstyle{remark}
\newtheorem{remark}[theorem]{Remark}

\newcommand{\norm}[1]{\left\lVert#1\right\rVert}

\newcommand{\x}{\mathbf{x}}
\newcommand{\y}{y}
\newcommand{\ysh}{y_{\Sigma}}
\newcommand{\bb}{\mathbf{b}}
\newcommand{\p}{\mathbf{p}}
\newcommand{\pu}{\mathbf{p}}
\newcommand{\uu}{\mathbf{u}}

\newcommand{\w}{\omega}
\newcommand{\e}{\mathrm{e}}

\usepackage[textsize=tiny]{todonotes}

\icmltitlerunning{PAC bounds of continuous  Linear Parameter-Varying systems related to neural ODEs}

\begin{document}

\twocolumn[
\icmltitle{PAC bounds of continuous Linear Parameter-Varying systems related to neural ODEs}



\icmlsetsymbol{equal}{*}

\begin{icmlauthorlist}
\icmlauthor{Dániel Rácz}{sztaki,elte}
\icmlauthor{Mihály Petreczky}{lille}
\icmlauthor{Bálint Daróczy}{sztaki}
%\icmlauthor{Martin Gonzalez}
%\icmlauthor{Firstname4 Lastname4}{sch}
%\icmlauthor{Firstname5 Lastname5}{yyy}
%\icmlauthor{Firstname6 Lastname6}{sch,yyy,comp}
%\icmlauthor{Firstname7 Lastname7}{comp}
%%\icmlauthor{}{sch}
%\icmlauthor{Firstname8 Lastname8}{sch}
%\icmlauthor{Firstname8 Lastname8}{yyy,comp}
\end{icmlauthorlist}

\icmlaffiliation{sztaki}{Institute for Computer Science and Control (ELKH SZTAKI),
 Budapest, Hungary}
\icmlaffiliation{elte}{Eötvös Loránd University, Budapest, Hungary}
\icmlaffiliation{lille}{Université de Lille, Lille, France}
%\icmlaffiliation{comp}{Company Name, Location, Country}
%\icmlaffiliation{sch}{School of ZZZ, Institute of WWW, Location, Country}

\icmlcorrespondingauthor{Dániel Rácz}{racz.daniel@sztaki.hu}
\icmlcorrespondingauthor{Mihály Petreczky}{mihaly.petrecszky@univ-lille.fr}
\icmlcorrespondingauthor{Bálint Daróczy}{daroczyb@ilab.sztaki.hu}
\icmlkeywords{Machine Learning, Neural ODE, LPV}

\vskip 0.3in
]



\printAffiliationsAndNotice{}  % leave blank if no need to mention equal contribution

\begin{abstract}

We consider the problem of learning Neural Ordinary Differential Equations (neural ODEs) within the context of Linear Parameter-Varying (LPV) systems in continuous-time. LPV systems contain bilinear systems which are known to be universal approximators for non-linear systems. Moreover, a large class of neural ODEs can be embedded into LPV systems. As our main contribution we provide Probably Approximately Correct (PAC) bounds under stability for LPV systems related to neural ODEs. The resulting  bounds have the advantage that they do not depend on the integration interval. 

%Our main contribution is to provide a PAC bound for learning LPV systems. 
%The obtained bounds have the advantage that they do not
%depend on the integration interval. Lastly we show how the obtained bound could be used to derive PAC bounds for neural ODEs.

%Motivated by the learning problem of Neural Ordinary Differential Equation (neural ODE) we consider the learning problem of linear parameter-varying (LPV) systems in continuous-time. LPV systems contain bilinear systems which are known to be universal approximators for nonlinear systems, including neural ODEs. Moreover, a large class of neural ODEs can be embedded into LPV systems. In addition, LPV systems are widely used in control. First, we provide a PAC bound for learning LPV systems. Second, we argue that the obtained bound could be used to derive PAC bounds for neural ODEs. The obtained bounds have the advantage that they do not depend on the integration interval.

\end{abstract}

\section{Introduction}
In this paper we consider the learning problem
for a class of ODEs called LPV systems \cite{toth2010modeling,MS12} and connect them to 
Neural Ordinary Differential Equations (neural ODE) 
\cite{chen2018neural,weinan2017proposal}. The inputs of such systems
are functions of time, and the can be divided into two categories:
control input and scheduling signal. The output is linear in the control input, but it is non-linear in the scheduling signal.
The differential equations which describe the system are linear
in the state and the control input, and they are non-linear in
the scheduling signal.
%\par
This class of ODEs represent a bridge between linear ODEs and 
non-linear ones, and they are widely used in control theory
\cite{toth2010modeling,MS12}. The reason behind their popularity is
the ability to model highly non-linear phenomena while allowing much simpler theory.
In particular, LPV systems include bilinear systems as a special class. In turn, the 
latter system class is a universal approximator for a large class of 
 nonlinear systems including many commonly used classes of neural ODEs.
In addition, many classes of neural ODEs  can be rewritten as LPV systems
by choosing a suitable class of scheduling signals. In particular, neural ODEs with
Rectified Linear Unit (ReLU) activation functions can be viewed as LPV systems, where the scheduling signal
represents the activation regions of the underlying ReLU network. 
%\par
Neural ODEs have gained popularity in the recent years, see e.g., \cite{massaroli2020dissecting,kidger2022neural,chen2018neural,weinan2017proposal} and the references therein.
There is a strong connection between Residual Networks (ResNet) and discrete-time neural ODEs, see e.g. \cite{wong2018provable}. A recent result \cite{sander2022residual} indicate that within the context of discrete-time ODE and ReLU activations the connection is more tight and the achievable accuracy is very similar, well above the threshold of comparable performance. 

As the neural ODEs originate from the field of dynamical system it is worth to discuss the plausible advantage of these methods. In comparison to traditional deep learning, where the models are trained with some form of back propagation on a fixed network structure and therefore the networks have constant depth, neural ODEs avoid the problem of vanishing gradients and the their variable depth can interpreted are not even discrete. 
%Neural
%Ordinary Differential Equations (NODE) \cite{chen2018neural,weinan2017proposal}.
%Although the adjoint method for optimization was introduced in
%\cite{chen2018neural} the first competitively performing NODE was shown in
%\cite{dupont2019augmented} where the authors exploited the input space with an
%augmentation method to a higher dimensional state space. Their experiments show
%that this step can be crucial achieving higher quality predictions.
%\par
The idea of studying neural LPV as a system for neural ODE is motivated by the wonderful properties of ReLU networks, namely that the structure of convex polytopes
formed in the input space by a ReLU network will be non-exponentially complex in the
number of neurons. More so that under some realistic conditions, e.g. Gaussian initialization
of the parameters, non-memorization problem or no-noisy target, the number of
convex activation regions is in expectation not growing exponentially with the depth of the network. This
property was theoretically and empirically shown in \cite{hanin2019deep} and recently the
authors in \cite{10095698} described a practically feasible algorithm to
enumerate the number of regions. Additionally we may exploit this phenomena by
bounding the number of linear transformations in the LPV system. 
%\par
Our main contribution is a PAC bound, Theorem~\ref{main}, for LPV system, using
Rademacher complexity and assuming stability of the underlying systems. 
The obtained error bound does not depend on the integration 
interval. In addition, we present some preliminary results on applying the obtained bounds to prove our PAC bounds for neural ODEs which can be represented by LPV systems. 
In particular, this extension may apply to neural ODEs with ReLU activation function,
under suitable assumptions on the activation regions. 
%\par

\section{Related results}

There are several recent advancements in the literature of neural ODEs. The main
challenges include robustness, capacity, stability constraints while maintaining
performance, e.g., in \cite{KangPaperStab1,rodriguez2022lyanet,HuangPaper,zakwan2022robust} stability and robustness of neural ODEs were
investigated. For instance, in 
\cite{rodriguez2022lyanet} the authors suggest a novel loss function
corresponding to Lyapunov stability. In their experiments the classification
performance increases significantly. 
%Since the stability is related to the
%µgeneralization gap we argue that in seek for competitive neural ODEs
%we need to balance and even accept marginally stable systems if the empirical
%µloss is acceptable given the generalization gap governed by the Rademacher
%complexity of our stable or marginally stable system.
%\par
Generalization gaps for neural ODEs were investigated in 
\cite{hanson2021learning,fermanian2021framing,marion2023generalization} and earlier
in \cite{sontag1998learning,KOIRAN199863,kuusela2004learning}.
The paper \cite{hanson2021learning,sontag1998learning,KOIRAN199863}
used Vapnik-Chervonenkis (VC) dimensions and approximations of the input and output signals by their high-order derivatives while stability was not assumed. As a result, the PAC bounds
contain terms which depend on the ability of the hypothesis class to approximate 
input-output pairs by their derivatives, and the bounds are exponential in the integration time and in the order of derivatives.  In contrast, in this paper we use Rademacher complexity instead of VC dimensions and the resulting bounds are not exponential in the integration time and they do not involve approximating continuous-time input and output signals by their derivatives. In \cite{kuusela2004learning} PAC bounds for linear ODEs were derived for classification. In contrast we consider regression despite a more general class of systems. Note that LPV systems include linear ODEs as a special case \cite{kuusela2004learning}. 
In \cite{fermanian2021framing,marion2023generalization} path controlled neural ODEs were considered and PAC bounds using Rademacher complexity were derived. Since stability was not assumed in \cite{fermanian2021framing,marion2023generalization}, these bounds depend exponentially on the integration time, and their system theoretic interpretation is not
obvious. In contrast, the bounds of this paper can be interpreted in terms of the well-known $H_2$ norm of LPV systems.
However, the class of systems in \cite{fermanian2021framing,marion2023generalization}
is more general than that of LPV systems. Note that bilinear and linear systems
are both subclasses of LPV systems and are special cases of systems
\cite{fermanian2021framing,marion2023generalization,hanson2021learning,sontag1998learning}.
%\par
There is a rich body of literature on finite-sample bounds for learning discrete-time dynamical systems from time-series \cite{vidyasagar2006learning,simchowitz2019learning,foster2020learning,Pappas1,SarkarRD21,lale2020logarithmic,campi2002finite,oymak2021revisiting}, but the learning problem considered in those papers is different from the one in this paper.

\section{Problem statement}

Linear parameter-varying (LPV) systems are special linear systems where the
coefficients of the linearities - referred to as \textit{scheduling signal} -
are time-varying functions. Concretely, an LPV system has the form
\begin{align}
    \label{system:2}
    \begin{cases}
        \dot{\x}(t) = A(\p(t)) \x(t) + B(\p(t))\uu(t) + \bb(\p(t)), \\
        \x(0)=0, \\
        \y_{\Sigma}(\uu, \p)(t) = C(\p(t)) \x(t), \\
    \end{cases} 
\end{align}
where $\x(t) \in \mathbb{R}^{n_x}$ is the state vector, $\uu(t) \in \mathbb{R}^{n_{in}}$ is the input and $\y_{\Sigma}(\p, \uu)(t) \in \mathbb{R}^{n_{out}}$ is the output of the system for all $t \in [0, T]$. The vector $\p(t) = (p_1(t), \dots, p_{n_p}(t))^T \in \mathbb{P} \subseteq
 \mathbb{R}^{n_p}$ is the scheduling variable. The matrices belonging to the system are defined as
\begin{align*}
    A(\pu(t)) &= A_0 + \sum\limits_{i=1}^{n_p}p_i(t)A_i \\
    B(\pu(t)) &= B_0 + \sum\limits_{i=1}^{n_p}p_i(t)B_i \\
    C(\pu(t)) &= C_0 + \sum\limits_{i=1}^{n_p}p_i(t)C_i \\
    \bb(\pu(t)) &= \bb_0 + \sum\limits_{i=1}^{n_p}p_i(t)\bb_i,
\end{align*}
such that the matrices $A_i \in \mathbb{R}^{n_x \times n_x}, B_i \in
 \mathbb{R}^{n_x \times n_{in}}, C_i \in \mathbb{R}^{n_{out} \times n_x}$ and the vectors $\bb_i \in  \mathbb{R}^{n_x}$ are time-independent and constant.
 Furthermore, we require $\x(t)$ to be an absolutely continuous function of the time domain, while $\y_{\Sigma}(\uu,\p)(t), \uu(t)$ and $\p(t)$ should be piecewise continuous. Note that for given $\uu$ and $\p$ the system admits a solution $(\x, \y_{\Sigma}(\uu,\p), \uu, \p)$ satisfying \eqref{system:2}.
 In principle we could have allowed the initial state $\x(0)$ to be a non-zero vector, but for the sake of simplicity we prefer to assume that  $\x(0)=0$.
 %We omit the notation of dependence on the initial state $x(0)$ as we consider systems which always start from the zero state.
 \begin{assumption}[Scalar output]
     In the rest of the paper we work with systems with scalar output, i.e. let $n_{out} = 1$.
 \end{assumption}
LPV systems contain linear control systems and bilinear control systems as special cases \cite{isidori1985nonlinear}. Indeed, by taking $\uu=0$, \eqref{system:2} becomes a bilinear system. In turn, it is known that bilinear systems can be used to approximate any non-linear control systems \cite{KrenerApprox}.
\begin{remark}
    \label{remark:1}
     Without loss of generality we can assume $\bb(\p(t)) = 0$ for all $t \in [0, T]$ as otherwise we concatenate an extra coordinate to $\uu(t)$ with value $1$ and replace $B_i$ with $[B_i | \bb_i]$, i.e we append $\bb_i$ to the columns of $B_i$ for all $i$.
 \end{remark}

\begin{assumption}
    \label{ass1}
    The scheduling signal $\p(t) \in [-1, 1]$ for all $t \in [0, T]$.
\end{assumption}
%\textbf{TODO}: NeuraLODE, resnet / rnn,
% linearizalhatosag, atvezeto szoveg. analogous to neural nwtworks see \cite{abbas2008polytopic}
%\textbf{TODO}: relu, sigmoidra miert igaz. Valojaban nem lesz igaz, mert $p$
%input lesz :(

%\textbf{TODO} MSE-re nem jo, MAE-ra igen, ref az a cikk

%Stability assumptions
%\begin{assumption}
%    \label{ass3}
%    There exists a square matrix $P$ with appropriate size and $\lambda > 0$
%    such that for all $i$
%    \begin{align*}        
%        A_i^T P + P A_i \preceq - \lambda P
%    \end{align*}
%\end{assumption}

In the sequel, we fix a family $\mathfrak{L}$ of LPV systems of the form \eqref{system:2}.
For this family of LPV systems, we make the following stability assumptions. 
%Next, we state an additional assumption about the stability of our system. 
\begin{assumption}
    \label{ass4}
There exists  $\lambda \ge  n_p$ such that
for any $\Sigma \in \mathfrak{L}$
of the form \eqref{system:2} there exists
$Q \succ 0$   such that 
    \begin{align*}        
        A_0^T Q + Q A_0 + \sum\limits_{i=1}^{n_p} A_i^TQA_i + 
        \sum\limits_{i=1}^{n_p} C_i^T C_i + C_0^T C_0 \prec  \lambda Q
    \end{align*}
\end{assumption}
%Our final assumption is about $A$. 

Due to our demand of $\uu(t)$ being piecewise continuous, for all $T$, 
$\sup\limits_{t \in [0, T]}\norm{\uu(t)}_2$ and 
$\norm{\uu(t)}_{L_2([0, T], \mathbb{R}^{n_{in}})}$ are finite.

%positive constants satisfying the following bounds.
At last we state a crucial lemma based on the previous assumptions which also implies additional useful properties of our system. 

\begin{lemma}
\label{ass4:lemma}
 If Assumption \ref{ass4} holds, 
 then for any $\Sigma \in \mathfrak{F}$,  let us define the following:
 for any $k > 0$, $i_1,\ldots,i_k \in \{1,\ldots,n_p\}$, for any
 $i_u,i_y \in \{0,\ldots,n_p\}$,
 \begin{align*}
 & w_{i_u,i_y}(t)=C_{i_y}e^{A_0t}B_{i_u} \\
 & w_{i_1,\ldots,i_k,i_y,i_y}(t,\tau_1,\ldots,\tau_k) =  \\      
 & C_{i_y} e^{A_0 (t-\tau_k)} A_{i_k}e^{\tau_k-\tau_{k-1}} \cdots A_{i_1}e^{A_0 \tau_1}B_{i_u}
  \end{align*}
 Then the following holds: 
 \label{ass4:lemma:eq1}
  \begin{align*}
  &\|\Sigma\|_{\lambda,H_2}^2 = ( \sum_{i_y,i_u=0}^{n_p} 
  \int_0^{\infty} w_{i_y,i_u}(t)\\
   &+\sum_{k=1}^{\infty} \sum_{i_1,\ldots,i_k=1}^{n_p}\int_0^{\infty} \int_0^{t} \int_0^{\tau_k} \cdots \int_0^{\tau_2} \\
   &\| w_{i_1,\ldots,i_k,i_y,i_y}(t-\tau_k,\ldots,\tau_1)\|^2_2 e^{\lambda t} d\tau_1\cdots d\tau_k dt )\\
   &=\sum_{i_u=0}^{n_p} \mathrm{trace} (B_{i_u}^TQB_{i_u}) < +\infty
 \end{align*}
 Moreover, for any $p \in \mathcal{P}$, $u \in \mathcal{U}$, 
 \begin{equation}
 \label{ass4:lemma:eq2}
    |y_{\Sigma}(\uu,\p)(t) | \le \|\Sigma\|_{\lambda,H_2} %\sqrt{n_{out}} 
    \|\uu\|_{L_2([0,T],\mathbb{R}^{n_{in}})} 
 \end{equation}
\end{lemma}
The proof can be found in Appendix~\ref{app_1}. 
Lemma \ref{ass4:lemma} states that any LPV system $\Sigma$ from $\mathfrak{L}$ have a finite 
weighted \emph{$H_2$ norm}
$\|\Sigma\|_{\lambda,H_2}$.
Indeed, $\|\Sigma\|_{0,H_2}$
is just the classical $H_2$ norm
\cite{gosea2021reduced} for LPV systems. For bilinear systems,
$\|\Sigma\|_{0,H_2}$ coincides with the $H_2$ norm defined in \cite{zl02}.
Intuitively, the $H_2$ norm is an upper bound on the peak output for inputs which have unit energy \cite{AntoulasBook}.
The quantity $\|\Sigma\|_{\lambda,H_2}$ can be thought of as a weighted version of the $H_2$ norm, where the past contribution of inputs is multipled by an exponential forgetting factor.  Note that if $\Sigma$
is of the form \eqref{system:2}
and we replace $A_0$ by $A_0+\frac{\lambda}{2}I$, then
the classical $H_2$ norm of the obtained system coincides with the weighted norm $\|\Sigma\|_{\lambda,H_2}$.

%\par 
Now let us consider the following learning problem for
systems of the form \eqref{system:2}.
To this end, let $\mathbf{S} = \{\p_i, \uu_{i},
\y_{i})\}_{1 \leq i \leq N}$, drawn independently from some unknown
distribution $\mathcal{P} \times \mathcal{D}  = (\mathcal{P}, \mathcal{U}, \mathcal{Y})$,
 along with an elementwise loss function $\ell:
\mathbb{R} \times \mathbb{R} \rightarrow \mathbb{R}$. The learning problem assumes the availability of 
i.i.d. input trajectories. The
\textit{empirical risk} is defined for a system $\mathcal{S}$ at a fixed timestep $T$
 as $\mathcal{L}_{N}(\Sigma) = \frac{1}{N}\sum\limits_{i = 1}^{N}
\ell(\y_{\Sigma}(\p_i, \uu_i)(T),y_i)$, while the \textit{true
risk} is defined as
 $\mathcal{L}(\Sigma) = \mathbb{E}_{(\mathcal{P},\mathcal{U},\mathcal{Y})}[\ell(
\y_{\Sigma}(\mathcal{P},\mathcal{U})(T),\mathcal{Y})]$. Our main interest is in the generalization gap
$\sup_{\Sigma \in \mathfrak{L}} |\mathcal{L}(\Sigma)-\mathcal{L}_N(\Sigma)|$. 
\par
Finally, we state the following assumptions on the data distribution and hypothesis class $\mathfrak{L}$.
\begin{assumption}
    \label{ass2}
    The elementwise loss function $\ell$
    is $K_{\ell}$-Lipschitz-continuous and $\ell(y,y)=0$.
   % \begin{align*}
   %     \ell(\y_i, \mathbf{\hat{y}}(\psi(\uu_i), T))
   %      = \hat{\ell}(\y_i - \hat{\y}(\psi(\uu_i), T)),
   % \end{align*}
   % where $\hat{\ell}: \mathbb{R}^{n_{out}}  \rightarrow \mathbb{R}$ is 
   % $L_{\hat{\ell}}$-Lipschitz continuous.
\end{assumption}
\begin{assumption}
%\begin{remark}
\label{ass6}
There exist constants 
$L_{\uu} > 0$, $K_{\uu} > 0$
$c_1 > 0$, $c_2 > 0$ such that
the following holds:
\[ 
    \sup_{\Sigma \in \mathfrak{L}} \|\Sigma\|_{\lambda,H_2} = c_1,
\]
and for almost any
sample $(\p,\uu,y)$
from $\mathcal{P} \times \mathcal{D}$ 
    \begin{align*}
        & \sup\limits_{t \in [0, T]}\norm{\uu(t)}_2 \leq K_{\uu} \\
        & \norm{\uu(t)}_{L_2([0, T],\mathbb{R}^{n_{in}})} \leq L_{\uu} \\
        & |y| < c_2, %\quad  
        %c_1 L_{\uu} < c_2.
    \end{align*}
\end{assumption}

%\begin{lemma}
%    There exists $L_{\Phi}$ positive constant such that the following is true
%    for all $\uu$ and $t \geq \tau$.
%    \begin{align*}
%        \norm{\Phi^{\uu}(t, \tau)} \leq \e^{L_{\Phi}(t - \tau)}
%    \end{align*}
%\end{lemma}
%
%\begin{proof}
%    Let $L_{\Phi} := \sup_{\norm{z} = 1}|\mathbf{z}^T A(\pu(t)) \mathbf{z}|$. We have
%    \begin{flalign*}
%        &\norm{\dot{\Phi}^{\uu}(t, \tau)} \leq L_{\Phi} \norm{\Phi^{\uu}(t, \tau)}&
%    \end{flalign*}
%
%    Now let $g(t) := \mathbf{z}^{T}\Phi^{\uu}(t, \tau)^T \Phi^{\uu}(t, \tau)
%    \mathbf{z}$ for some $\mathbf{z}$ with $\norm{\mathbf{z}} = 1$.
%   Then
%    \begin{flalign*}
%        &\dot{g}(t) = 2 \mathbf{z}^T \Phi^\uu (t, \tau)^T A(\pu(t))
%        \Phi^\uu (t, \tau) \mathbf{z} \leq 2 L_{\Phi} g(t)&
%    \end{flalign*}
%    By Grönwall's lemma we have
%    \begin{flalign*}
%        &g(t) \leq g(\tau) \e^{2L_{\Phi}(t - \tau)},&
%    \end{flalign*}
%    which yields to 
%    \begin{flalign*}
%        &\norm{\Phi^\uu(t, \tau)} \leq \e^{L_{\Phi}(t - \tau)}&
%    \end{flalign*}
%    by the definition of $g(t)$.
%\end{proof}
%
%\begin{lemma}
%    Under \textbf{TODO FELTETELEK} the following hold.
%        \begin{tabular}{p{3cm}p{3cm}p{3cm}}
%    \begin{align*}
%        \norm{C(\pu(t))} &< L_{C} \\
%        \norm{B(\pu(t))} &< L_{B} \\
%        \norm{b(\pu(t))} &< L_{b}
%    \end{align*}
%
%\end{lemma}

\section{Rademacher complexity of LPVs}

In this section we bound the Rademacher complexity of our LPV system under the conditions we stated in the last section. First we define the Rademacher complexity. 

\begin{definition}(e.g. see definition 26.1 in \cite{shalev2014understanding})
    \label{defiradem}
    The Rademacher complexity of a bounded set
     $\mathcal{A} \subset \mathbb{R}^{m}$ of vectors is defined as
    \begin{align*}
        R(\mathcal{A}) = \mathbb{E}_{\mathbf{\sigma}}\Bigg[\sup_{a \in \mathcal{A}}
        \frac{1}{m} \sum\limits_{i = 1}^{m}\sigma_i a_i \Bigg],
    \end{align*}
    where the random variables $\mathbf{\sigma}_i$ are i.i.d such that 
    $\mathbb{P}[\sigma_i = 1]  = \mathbb{P}[\sigma = -1] = 0.5$. The Rademacher complexity of a set of functions $\mathcal{F}$ over a set of
    samples $S = \{s_1\dots s_m\}$ is defined as
    \begin{align*}
        R_{S}(\mathcal{F}) = R(\left\{(f(s_1),\dots,f(s_m)) \middle
         | f \in \mathcal{F} \right\})
    \end{align*}
\end{definition}

Our main contribution is a bound to the Rademacher complexity derived from the following theorem under
Assumptions \ref{ass1}, \ref{ass2}, \ref{ass4} and \ref{ass5}. 

\begin{theorem}(e.g. see Theorem 26.5 in \cite{shalev2014understanding})
    \label{pac}
    Let $\mathcal{E}$ be a compact set of hypotheses and $T$ is a timestep.
    For any $\delta \in ]0, 1[$ and  we have
    \begin{align*}
        \mathbb{P}_{\mathcal{S}}\Bigg(\forall \Sigma \in \mathcal{E}:
        \mathcal{L}(\Sigma) - \mathcal{L}_{N}(\Sigma)
         \leq R(T, N, \delta) \Bigg) \geq 1 - \delta,
    \end{align*}
    %where $A_0(\Sigma) = \{ (\hat{\ell}(\hat{\y}(\psi(\uu_1), T) - y_1)
    %,\dots,\hat{\ell}(\hat{\y}(\psi(\uu_N), T) - y_N))^T
    % | \Sigma \in \mathcal{E}\}$.
    where
    \begin{align*}
    &R(T, N, \delta) = 2 R_{S}(L_{0}(\Sigma, T))
         + B(T)\sqrt{\frac{2 log (\frac{4}{\delta})}{N}} \\
     &L_0(\Sigma, T) = \{ (\ell(\ysh(\uu_N,\p_N)(T), y_1),  \\
    &\hspace{60pt} \dots,\\
    &\hspace{60pt} \ell(\ysh(\uu_N,\p)(T), y_N))^T %\\&\hspace{60pt}
    |\text{ } \Sigma \in \mathfrak{F} \} 
    \end{align*}
    and $B(T)$
    %$ \leq 2K_{\ell}c_2$ 
    is an upper bound on the loss at time $T$.
\end{theorem}

Now we state our main theorem. 

\begin{theorem}
    \label{main}
    %Let $\mathfrak{L}$ be a compact family of models
    % in the form of system \ref{eq:system}
    %together with the 
    Under assumptions \ref{ass1}, \ref{ass2}, \ref{ass4} and \ref{ass6} and let $T$ be a
    timestep. Then for
    for
    any $\delta \in ]0, 1[$, Theorem \ref{pac}
    holds with the bound of \\
     \[ R(T, N, \delta) = \frac{c_3}{\sqrt{N}}
    \Big(2 + \sqrt{2\log\big(\frac{4}{\delta}\big)}\Big) 
    \]
    where
    \[ 
       c_3=\max\{2K_{\ell}c_2,
       2K_{\ell}c_1 L_{\uu}(n_p+1) \}
    \]
\end{theorem}
The bound in Theorem \ref{system:2} depends on the
maximal (weighted) $H_2$ norm of
the LPV systems from $\mathfrak{L}$,  the maximal energy of the inputs $\uu$ and 
on the maximal possible value of the outputs $y$. The smaller the $H_2$ norms are of the LPV systems involved the smaller is the generalization gap. Likewise, the smaller the input energy is, the smaller is the generalization gap.
Note that the bound does not depend on $T$. However, the constant $L_{\uu}$ may depend on $T$, if we consider inputs $\uu$ for which the global energy
$\int_0^{\infty} \|u(s)\|_2^2ds$
is either infinite or too large.

The computation of the upper bound of Theorem \ref{main}
involves estimating the 
constant $c_1$ which represents the maximal $H_2$ norms of the elements of $\mathfrak{L}$.
In practice, this not always easy. Below we present some additional 
assumption which allows us to compute the bounds of Theorem \ref{main}.
%The proof of the main theorem involves several lemmas which we now state before we prove the theorem at the end of this section. 
\begin{assumption}
    \label{ass5}
    There exists a positive real number $\gamma > 0$ and $\Gamma > 0$ such that
     $\gamma \geq \Gamma+n_p$
    and for every 
    $\Sigma \in \mathfrak{L}$
    of the form \eqref{system:2},
    $\Gamma \ge \sup_{i=1}^{n_p} \norm{A_i}^2n_p$ and 
    \begin{align*}        
        \norm{\e^{A_0 t}} \leq \e^{-\frac{\gamma}{2} t}
    \end{align*}
\end{assumption}
\begin{lemma}
\label{lemma:sigma}
 If Assumption \ref{ass5} holds then Assumption \ref{ass4}
 holds  with any $n_p \le \lambda \le \gamma-\Gamma$ and $Q$ being equal to the identity matrix, and 
 \begin{align*}
     c_1 \le K^2_{\w} &:= 
         (n_p+1)^2 K^2_C K^2_B
          \Big(\frac{1}{\lambda}  + \frac{1}{\gamma - \lambda + \Gamma } \Big)
 \end{align*}
 where $K_B=\sup_{i=1}^{n_p} \|B_i\|$, 
 $K_C=\sup_{i=1}^{n_p} \|C_i\|$.
\end{lemma}
The proof can be found in Appendix \ref{app_3}.

\subsection*{Proof of Theorem \ref{main}}
In order to prove Theorem \ref{main} - according to Theorem \ref{pac} -
 it is sufficient to find an upper bound to the Rademacher complexity of
  $L_0(\Sigma, T)$ and to the loss function $l$ at time $T$. Before we prove the main theorem at the end of this section we will state a lemma about the upper bound of the Rademacher complexity. To identify the bound we reformulate our system and show that our system is bounded.
 The following lemma is another key component in our proof which states that our output is an inner product in a Hilbert space.
 
\begin{lemma}
   \label{lemma:main} 
   There exists a Hilbert space $\mathcal{H}$ such that
    \begin{align*}
        \y_{\Sigma}(\uu, \p)(t) &= \langle \w^{T, \Sigma},
        \varphi^{T, \uu,\p} \rangle_{\mathcal{H}}
    \end{align*}
    where $\w^{T, \Sigma}$ does not depend on $\uu$ and $\p$, $\varphi^{T, \uu,\p}$
    does not depend on $\Sigma$ and they are respectively elements of 
     the space $\mathcal{H}$.\\
    Furthermore $\norm{\varphi^{T, \uu, \p}}_{\mathcal{H}}
           \leq L_{\uu} (n_p + 1)$
    and $\norm{\w^{T, \uu}}_{\mathcal{H}}
           \leq \norm{\Sigma}_{\lambda, H_2}$.
     %for any choice of $\lambda$ satisfying $n_p + 1 < \lambda < \gamma - n_p - 1$.
\end{lemma}
The proof can be found in Appendix \ref{app_2}.

Next we state a useful lemma about the Rademacher complexity of Lipschitz continuous functions. 

\begin{lemma}(e.g. see Lemma 26.9 in \cite{shalev2014understanding})
    \label{helper1}
    Under Assumption \ref{ass2} for all $\p$
    \begin{align*}
        R_{S}(L_0(\Sigma, T)) \leq K_{\ell} \cdot
         R_{S}(\{ S(\Sigma, T)
          | \Sigma \in \mathcal{E}\}) 
    \end{align*}
    where $S(\Sigma, T) = (\ysh(\uu_1, \p)(T),\dots,\ysh(\uu_N, \p)(T))^T$.
\end{lemma}

Next we state a crucial lemma about the upper bound of the complexity given bounded linear transformations. 

\begin{lemma}(e.g. see Lemma 26.10 in \cite{shalev2014understanding})
    \label{helper2}
    Let $S = (\x_1,\dots,\x_m)$ be vectors in a Hilbert space. Then 
    \begin{align*}
        R(\{\langle \mathbf{w}, \x_1 \rangle , \dots ,
          \langle \mathbf{w}, \x_m \rangle)^T | \norm{\mathbf{w}}_2 \leq 1 \})
          \leq \frac{\max_{i} \norm{\x_i}_2}{ \sqrt{m}}
    \end{align*}
\end{lemma}

Based on the previous results we can upper bound the empirical Rademacher complexity of $L_0(\Sigma,T)$. 

\begin{lemma}
    \label{rademacher}
    For $c_4=L_{\uu}K_{\ell}c_1 (n_p + 1) $,
    %There exists a positive real constants $c_4$ such that
    \begin{align*}
        R_{S}(L_0(\Sigma, T)) \leq \frac{c_4}{\sqrt{N}}
    \end{align*}
\end{lemma}
\begin{proof}
    Let $R = R_S(\{ S(\Sigma, T) | \Sigma \in \mathcal{E}\})$
    where
    $S(\Sigma, T) = (\langle \w^{T, \Sigma},
      \varphi^{T, \uu_1, \p} \rangle_{\mathcal{H}},
      \dots,
      \langle \w^{T, \Sigma},
      \varphi^{T, \uu_N, \p} \rangle_{\mathcal{H}}
    )^T$ based on Lemma \ref{lemma:main}.
    Using Lemma \ref{helper1} and Lemma \ref{helper2} we have
    \begin{align*}
     R_{S}(L_0(\Sigma)) &\leq K_{\ell}
         R_{S}(\{ S(\Sigma, T) | \Sigma \in \mathcal{E}\})
      \leq K_{\ell} R\\ 
      &\leq \frac{K_{\ell}}{\sqrt{N}} L_{\uu} (n_p + 1) c_1
    \end{align*}
    therefore $c_4 = L_{\uu}K_{\ell}c_1 (n_p + 1) $.
\end{proof}

The final element to apply Theorem \ref{pac} is to bound $B(T)$. 

\begin{lemma}
    \label{bt}
    Under Assumption \ref{ass6}% $K_l$ and $c_2$ such that
    \begin{align*}
        B(T) \leq 2 K_{\ell} \max\{c_2,L_{\uu}c\}
    \end{align*}
\end{lemma}
\begin{proof}[Proof of Lemma \ref{bt}]
    Due to Assumption \ref{ass2} we have
    \begin{align*}
        |\ell(\ysh(\uu, \p)(T), y)| &\leq 2K_{\ell}\max\{|\ysh(\uu, \p)(T)|,|y|\} \\
        &\leq 2K_{\ell}\max\{L_{\uu}c_1,c_2\} %\le 2 K_{\ell}c_2
    \end{align*}
    The last inequality follows from applying Lemma \ref{ass4:lemma}.
    %to
    %$\ysh(\uu, \p)(T)$ followed by a Cauchy-Schwartz inequality along with
    %Lemma \ref{lemma:sigma}. Therefore the statement is true with $c_5 = 2K_{\ell}c_2$ from the previous Lemma.
\end{proof}

Finally, we can prove the main theorem. 

\begin{proof}[Proof of Theorem \ref{main}]
    The theorem follows from Lemmas \ref{rademacher} and \ref{bt}
    together with Theorem \ref{pac} for 
    $B(T)= L_{\uu}K_{\ell}K_{\w} (n_p + 1)$.
\end{proof}
\section{Connection to VC-dimension and Neural ODEs}

 A major motivation for studying LPV systems is that they
 can be used for embedding neural ODEs. 
 More precisely,  we consider the following dynamical system 
\begin{align}
    \label{eq:system}
    \mathcal{S} 
    \begin{cases}
        \dot{\x}(t) = f(\x(t), \uu(t)) \\
        %\x(0) = \psi(\uu) \\ %\Label{system} \\
        y(\x(0), t) = h(\x(t))
    \end{cases}
\end{align}
%\begin{flalign}
%    \label{eq:system}
%    \begin{cases}
%        \mathbf{\dot{x}}(t) = f(\x(t), \uu(t)) \\
%        \x(0) = \psi(\uu) \\ %\Label{system} \\
%        \mathbf{\hat{y}}(\x(0), t) = h(\x(t))
%    \end{cases} &&
%\end{flalign}
where $\x \in \mathbb{R}^{n_{x}}$, $\uu \in
\mathbb{R}^{n_{in}}$ and 
$y \in \mathbb{R}$ respectively
denote the state, input and output vectors. We refer to $\psi:
\mathbb{R}^{n_{in}} \rightarrow \mathbb{R}^{n_{x}}$ and $h: \mathbb{R}^{n_x}
\rightarrow \mathbb{R}$ as the input and output transformations
corresponding to the first and last layer of regular deep networks.
%Furthermore, we assume the function $f \in \mathcal{F}$ to be $L_{f}$-Lipschitz
%continuous in $\x$, implying the system to admit a unique solution by 
%the Picard-Lindelöf theorem \textbf{TODO: ref, Misi?}. The system is denoted by
%$\Sigma = (f, \psi, h)$ and assumed to originate from a hypothesis set 
%$\mathcal{E}$. \\
%\par
A common choice of $f$ and $h$ are (deep) neural networks.
For a wide variety of choice of $f$ and $h$, \eqref{eq:system}
can be embedded into an LPV system as follows:
%The dynamical system in (\ref{eq:system}) is an LPV system, i.e 
there
    exists the square matrices $A_i, B_i, C_i$ and vectors $\mathbf{b}_i$ of appropriate
    size along with the functions $p_i^{\uu}: [0, T] \rightarrow [0, 1]$
    for all $\uu$, such that the following hold for all $t \in [0, T]$.
    \[
     \Sigma_{\mathcal{S}} \left \{
    \begin{split}
        &\dot{\x}(t) = A_0 \x(t) + B_0 \uu(t) + \mathbf{b}_0 \\
        & \hspace{18pt} +\sum\limits_{i = 1}^{n_p}
        p_i^{\uu}(t)(A_i \x(t) + B_i \uu(t) + \mathbf{b}_i) \\ 
        %&\x(0) = \psi(\uu)   \\ %\Label{system} \\
        &\y(t) = C_0  \x(t)
         + \sum\limits_{i=1}^{n_p}p_i^{\uu}(t)C_i\x(t) \\
         & |p_i^{\uu}| \le 1
      \end{split}   \right.
    \]
That is, for any input $\uu$ there exist a scheduling signal
$p^{\uu}=\begin{bmatrix} p_1^{\uu} & \ldots & p_{n_p}^{\uu} \end{bmatrix}^T$ such that the output of $\mathcal{S}$
corresponding to $\uu$ is the output of the LPV $\Sigma_{\mathcal{S}}$ corresponding to the input $\uu$ and 
scheduling $p^{\uu}$, i.e.
\[  y_{\Sigma_{\mathcal{S}}}(0,\uu,p^{\uu})(t)=\y(t). \]
For instance, $f$ and $h$ could be chosen as  $f(x,\uu)=\stackrel{\sigma}{\rightarrow}(Ax+B\uu)$, $h(x)=Cx$ with $\stackrel{\sigma}{\rightarrow}(z)=(\sigma(z_1),\ldots,\sigma(z_n))^T$ and $\sigma$ is one of the common activation functions, e.g., $ReLU$, $tanh$ etc. 
If $\sigma$ is ReLU, then $p^{\uu}(t)$ is the linear region
of $f$, i.e., $n_p=n$, $A_i,B_i$ are the $i$th rows of $A$ and $B$
and 
\[ p^{\uu}(t)=\left\{ \begin{array}{rl} 1 &  (Ax(t)+B\uu(t))_i \ge 0 \\ 0 & \mbox{otherwise} \end{array}\right.,
\]
where
$\x(t)$ is the solution of \eqref{eq:system}.
If $\sigma$ is $\tanh$, then again $n_p=n$ and $A_i$, $B_i$
are the $i$th rows of $A$ and $B$ and 
\[ p^{\uu}(t)=\left\{ \begin{array}{rl} \frac{\tanh((Ax(t)+B\uu(t))_i)}{(Ax(t)+B\uu(t)_i} & \mbox{if} (Ax(t)+B\uu(t))_i \ne 0 \\ 0 & \mbox{otherwise,} \end{array}\right. \]
see \cite{NeuralStab,verhoek2023learning}.
%\par
Note that the embedding of $\mathcal{S}$ into an LPV system is not unique, various choices will result in LPV systems with different properties. The construction of such embeddings is an active research area in the field of control of LPV systems 
\cite{verhoek2023learning,AbbasRolandMulti,KoelewijnThesis}.
%\par
Motivated by this remark, in this section we consider the following extended version of the LPV learning problem. 
We assume that there exists a set $\Phi$ of 
maps $\psi_{p}:L_2([0,+\infty),\mathbb{R}^{n_u}) \rightarrow PC([0,+\infty),[-1,1]^{n_p})$ and we consider a family
$\mathfrak{S}$ of systems of the form \eqref{eq:system} and a family of LPV systems $\mathfrak{L}$ such that for
any $\mathcal{S} \in \mathfrak{S}$ there exists an LPV $\Sigma_{\mathcal{S}}$  and  $\psi \in \Phi$ such that 
for any $\uu \in L_2([0,T]$,
$y_{\mathcal{S}}(\uu)=y_{\Sigma_{\mathcal{S}}}(0,\uu,p^{\uu})$
with $p^{\uu}=\psi(\uu)$. 

We assume that for any input $\uu$, the scheduling $p^{\uu}$
which generates the output of $\mathcal{S}$ as an output of an LPV system originates from one of the functions from $\Phi$.
That is, we can identify  systems \eqref{eq:system}  viewed as a hypothesis class as the sub-family of pairs 
$(\phi,\Sigma)$, $\phi \in \Phi$, $\Sigma \in \mathfrak{L}$.
Then we can use the results on the generalization gap for LPV systems for deriving a PAC bound for systems of the form \eqref{eq:system} which can be embedded into LPV systems.
%\par
More precisely, let us consider the learning problem defined previously.  
%for systems of the form \eqref{eq:system}.
%To this end, let $\mathbf{S} = \{(\uu_{i},
%\y_{i})\}_{1 \leq i \leq N}$, drawn independently from some unknown
%distribution $\mathcal{D} = (\mathcal{U}, \mathcal{Y})$,
% along with an elementwise loss function $\ell:
%\mathbb{R}^{n_{out}} \times \mathbb{R}^{n_{out}} \rightarrow \mathbb{R}$. The learning 
problem again assumes the availability of 
%i.i.d. input trajectories. The
%\textit{empirical risk} is defined for a system $\mathcal{S}$ at a fixed timestep $T$
% as $\mathcal{L}_{N}(\mathcal{S}) = \frac{1}{N}\sum\limits_{i = 1}^{N}
%\ell(\y_i, \y(\uu_i, T))$, while the \textit{true
%risk} is defined as
% $\mathcal{L}(\mathcal{S}) = \mathbb{E}_{(\mathcal{U},\mathcal{Y})}[\ell(\mathcal{Y},y(\psi(\mathcal{U}), T))]$. 
%Finally, we consider the function $f$
%to be parametric, i.e. it depends on some parameter vector $\theta \in
%\mathbb{R}^{P}$, and emphasize this fact by the notation $f = f_{\theta}$.
%We are again interested in generalisation gap
%\[ \sup_{\mathcal{S} \in \mathfrak{S}} \mathcal{L}(\mathcal{S})-\mathcal{L}_N(\mathcal{S}).
%\]
The main idea to bound the generalization gap is as follows. Let us define 
\[
  \begin{split}
     \mathcal{L}_{N}(\Sigma,\phi) &= \frac{1}{N}\sum\limits_{i = 1}^{N}
\ell(\y_i, y_{\Sigma}(\uu_i, p_i, T)), \\
      & p_i =\phi(\uu_i), i=1,\ldots,N \\
  \mathcal{L}(\Sigma,\phi) & = \mathbb{E}_{(\mathcal{U}, \mathcal{Y})} \ell(\mathcal{Y},
y_{\Sigma}(\mathcal{U},\phi(\mathcal{U}),T))
 \end{split}
\]
It then follows that for any $\mathcal{S} \in \mathfrak{S}$
there exists $\phi \in \Phi$ and $\Sigma \in \mathfrak{L}$, such
that
\[
   \mathcal{L}_{N}(\Sigma,\phi)=\mathcal{L}_N(\mathcal{S}), \quad 
   \mathcal{L}(\Sigma,\phi)=\mathcal{L}(\mathcal{S}).
\]
Therefore the generalization gap for systems from $\mathfrak{S}$ satisfies the following
\begin{equation}
\label{vc:eq1}
\begin{split}
  & \sup_{\mathcal{S} \in \mathfrak{S}} \mathcal{L}(\mathcal{S})-\mathcal{L}_N(\mathcal{S}) \le  \\
  & \sup_{\Sigma \in \mathfrak{L}, \phi \in \Phi}
   \mathcal{L}(\Sigma,\phi)-\mathcal{L}_N(\Sigma,\phi) 
\end{split}
\end{equation}
In order to estimate the latter generalization gap, we can use the
PAC bound derived for LPV systems. To this end, we need to assume
that $\Phi$ has a finite VC dimension. 
\begin{assumption}
\label{assum:vc}
 Assume that the $\Phi$ is either finite or for any $T > 0$, the set $\Phi_T=\{ (\uu,\phi(\uu)) \mid \uu \in L_2([0,T],\mathbb{R}^{n_{in}}), \phi \in \Phi \}$ has a finite VC-dimension $d_T$.
\end{assumption}

\begin{theorem}
\label{theo:vc}
 Under Assumption \ref{assum:vc} and the assumptions of Theorem
 \ref{main} for any $\delta \in [0,\delta)$
 \begin{align*}
        \mathbb{P}_{S}\Bigg(\forall \mathcal{S} \in \mathfrak{S}:
        \mathcal{L}(\mathcal{S}) - \mathcal{L}_{N}(\mathcal{S})
         \leq \bar{R}(T, N, \delta) \Bigg) \geq 1 - \delta,
    \end{align*}
 where $\bar{R}(T,N,\delta)$ is defined  as follows.
 If $\Phi$ is finite, then 
 \begin{align*}
  \bar{R}(T,N,\delta) &= \frac{2\mathrm{card}(\Phi)K_{\ell} c_1 L_{\uu} (n_p+1)}{\sqrt{N}} \\
  & +B(T)\sqrt{\frac{2\ln(2/\delta)}{N}}
 \end{align*}
 %where $\mathrm{card}(\Phi)$ is the cardinality of set $\Phi$ and $B$ is defined as follows. It is assumed that $\ell(y',y) \le B$ for all $y',y$ such that
 %$|y'| < R, |y| < R$ and 
 %we assume that for any data 
 %$\{(y_i,\uu_i)\}_{i=1}^{N}$ drawn from $\mathcal{D}$ it holds that
% \[
%    \begin{split}
%    & \| y_i| \le R, \quad  \|\uu_i\|_{L_2} \le L_u \\
%    & \sup_{\Sigma \in \mathfrak{L}}
%       \|\Sigma\|_{H_2,\lambda} % < \frac{R}{L_u \sqrt{n_{out}}}
%   \end{split}
% \]
where $B(T) \le 2K_{\ell}\max\{c_2, c_1L_{\uu}\}$.
 If $\Phi$ is possibly infinite, but the set $\Phi_T$ defined in Assumption \ref{assum:vc}
 has finite VC dimension $d_T$, then
 \begin{align*}
  & \bar{R}(T,N,\delta)= \sqrt{2}B(T)(2+\sqrt{\log(2eN/d_T)d_T}) + \\
  & \frac{2K_{\ell}c_1L_{\uu}(n_p+1)}{\sqrt{N}} + B(T) \sqrt{\frac{2\ln(2/\delta)}{\sqrt{N}}}  
 \end{align*}
\end{theorem}

The proof can be found in Appendix \ref{app_4}. The bound above suggests that the VC-dimension of the set of scheduling signals $\Phi$ together with the weights of the systems from $\mathfrak{L}$ determine the generalization power of system \eqref{eq:system}.
If $\Phi$ is a finite set, i.e., there are only finitely many possible different scheduling signals for each input, then  
the PAC bounds are $O(\frac{1}{\sqrt{N}})$. In comparison in case of classical deep ReLU networks the number of linear regions are always finite and under some realistic conditions the number of linear regions are not necessarily increasing exponentially with depth \cite{hanin2019deep}. If $\Phi$ is possibly infinite, but has a finite VC dimension, then the PAC bound is $O(\sqrt{\log(N)d_T)}$, where 
$d_T$ is the VC dimension of $\Phi$. 
%\par
This conclusion is especially interesting for neural ODEs which correspond to  neural networks with ReLU activation function. For such systems, the set $\Phi$ can be taken as the set of all sequences
of activation regions of the ReLU network.
There are some empirical %\cite{10095698} and theroretical \cite{hanin2019deep} 
evidence that ReLU networks tend to visit few regions as their depth grow. 
This means that neural ODEs which are represented by ReLU networks could also exhibit the same behavior, i.e., the set $\Phi$
would have a small VC dimension, or possibly it would be finite with a small cardinality.
All this remains a topic of future research.

\section{Discussion}

Viewing neural ODEs within the context of bilinear LPVs, we established novel PAC and Rademacher complexity bounds under stability conditions. The resulting PAC bounds do not depend on the integration interval. The generalization bounds indicate the importance of the scheduling signal in LPV systems and we intend to continue to examine how the scheduling affect the performance and the generalization gap. We believe that analogous results hold for marginally stable systems and for highly overparametrizated neural ODE systems. 


%\bibliography{example_paper}
\bibliography{references}
\bibliographystyle{icml2023}
%\bibliographystyle{plain}

%%%%%%%%%%%%%%%%%%%%%%%%%%%%%%%%%%%%%%%%%%%%%%%%%%%%%%%%%%%%%%%%%%%%%%%%%%%%%%%
%%%%%%%%%%%%%%%%%%%%%%%%%%%%%%%%%%%%%%%%%%%%%%%%%%%%%%%%%%%%%%%%%%%%%%%%%%%%%%%
% APPENDIX
%%%%%%%%%%%%%%%%%%%%%%%%%%%%%%%%%%%%%%%%%%%%%%%%%%%%%%%%%%%%%%%%%%%%%%%%%%%%%%%
%%%%%%%%%%%%%%%%%%%%%%%%%%%%%%%%%%%%%%%%%%%%%%%%%%%%%%%%%%%%%%%%%%%%%%%%%%%%%%%
\newpage
\appendix
\onecolumn
%\section{You \emph{can} have an appendix here.}

\section{Proof of Lemma~\ref{ass4:lemma}}
\label{app_1}

 If Assumption \ref{ass4} holds, then 
 $A_0^T Q + QA_0 \prec -\lambda Q$, and hence
 $A_0+\lambda I$ is Hurwitz. 
 Let $S= A_0^T Q + Q A_0 + \sum\limits_{i=1}^{n_p} A_i^TQA_i + 
        \sum\limits_{i=1}^{n_p} C_i^T C_i + C_0^T C_0+\lambda Q$.
 Then $S \prec 0$ and hence $S=-VV^T$ for some $V$. 
 Define $\tilde{C}=\begin{bmatrix} C_0^T & \ldots & C_{n_p}^T & V^T \end{bmatrix}^T$, $\tilde{A}=(A+0.5\lambda I)$ and $N_i=A_i$, $i=1,\ldots,n_p$. Then
 \[ 
   \tilde{A}^T Q + Q \tilde{A}+\sum_{i=1}^{n_p} N_i^TQN_i + \tilde{C}^T\tilde{C}=0
 \]
 and hence by \citet[Theorem 2]{zl02} 
 the $H_2$-norm  of the Volterra-series of
 the bilinear system 
 \begin{align*}
  & \dot z = \tilde{A}z+\sum_{i=1}^{n_p} \left(N_i z  p_i(t) + G_i p_i(t) \right)  \\
  & \tilde{y}=\tilde{C}z
 \end{align*}
 for any choice of the matrices $\{G_i\}_{i=1}^{n_p}$, i.e., for  
 \begin{align*}
     g^{i}(t)=\tilde{C}e^{\tilde{A}t}G_i
 \end{align*}
 and for all $i_1,\ldots,i_k=1,\ldots,n_p$, $k >0$
 \begin{align*}
   g^{i_1,\ldots,i_k,i}(t_k,\ldots,t_1)
   =\tilde{C}e^{\tilde{A}t_k}N_{i_k}\cdots e^{\tilde{A}t_1}N_{i_1}e^{\tilde{A}t_1}G_i
 \end{align*}
 the following infinite sum of infinite integrals is convergent:
 \begin{align*}
 & \sum_{i=1}^{n_p} \int_0^{\infty} \| g^i(t)\|_2^2 dt+  \sum_{k=1} \int_0^{\infty} \cdots \int_0^{\infty} 
 \|g^{i_1,\ldots,i_k,i}(t_k,\ldots,t_1)\|_2^{2} dt_1\cdots dt_k < +\infty
 \end{align*}
 By choosing $G_i=B_{i_u}$, $i=0,1,\ldots, n_p$
 it follows that $w_{i_1,\ldots,i_k,i_u,i_y}(t,\tau_1,\ldots,\tau_k)e^{\lambda t}$ 
 is the $(i_y+1)$th block row of
 $g^{i_1,\ldots,i_y}(t-\tau_k,\tau_{k}-\tau_{k-1},\ldots,\tau_2-\tau_1,\tau_1)$ and hence it follows that
 the right-hands  side of 
\begin{align*}
  \|\Sigma\|_{\lambda,H_2}^2 = \left ( \sum_{i_y,i_u=0}^{n_p} \sum_{k=0}^{\infty} \sum_{i_1,\ldots,i_k=1}^{n_p}\int_0^{\infty} \int_0^{t} \int_0^{\tau_k} \cdots \int_0^{\tau_2} \right.
   \left .\| w_{i_1,\ldots,i_k,i_y,i_y}(t-\tau_k,\ldots,\tau_1)\|^2_2 e^{\lambda t} d\tau_1\cdots d\tau_k dt \right) < +\infty
 \end{align*}
is convergent.

From Lemma \ref{lemma:main} it follows that 
$y_{\Sigma}(0,\uu,p)(T)$
can be represented by the scalar
product
$y_{\Sigma}(0,\uu,p)(T)=\langle \w^{T, \Sigma},
        \varphi^{T, \uu, \p} \rangle_{\mathcal{H}_2}$
in a suitably defined Hilbert-space. From the definition of $\w^{T,\Sigma}$ and  $\varphi^{T, \uu, \p}$ it follows that
$\|\w^{T,\Sigma}\|_{\mathcal{H}_2} \le \|\Sigma\|_{\lambda,H_2}$, and  $\|\varphi^{T,\uu, \p}\|_{\mathcal{H}_2} \le \|u\|_{L_2}$ 
and hence by the Cauchy-Schwartz inequality for any $p \in \mathcal{P}$, $u \in \mathcal{U}$
\begin{align*}
    |y_{\Sigma}(0,u,p)(t) | \le \|\Sigma\|_{\lambda,H_2} \|u\|_{L_2} 
\end{align*}
follows.

%From Lemma \ref{lemma:main} it follows that the $i$th coordinate
%$y_{\Sigma}(0,\uu,p)_i(T)$
%of $y_{\Sigma}(0,\uu,p)_i(T)$
%can be represented by the scalar
%product
%$y_{\Sigma}(0,\uu,p)_i(T)=\langle \w_{2}^{T, \Sigma,i},
%        \varphi_{2}^{T, \uu} \rangle_{\mathcal{H}_2}$
%in a suitably defined Hilbert-space.

%where 
%$\w_2^{T,\Sigma}=\{ \tilde{w}_{i_u,i_y}, \tilde{w}_{i_1,\ldots,i_k,i_u,i_y}\}_{k > 0, i_1,\ldots, i_k=1,\ldots,n_p, i_u,i_y =0,\ldots,n_p}$,
%where 
%$\tilde{w}_{i_u,i_y}:[0,T] \ni \tau \mapsto e^{\lambda T} w_{i_u,i_y}(\tau)$
%and 
%$\tilde{w}_{i_1,\ldots,i_k,i_u,i_y} \in %L_2([0,T]^{k+1},\mathbb{R}^{n_u})$  such that
%\begin{align*}
%& \tilde{w}_{i_1,\ldots,i_k,i_u,i_y}(\tau,\tau_k,\ldots,\tau_1)=\\
%& \left\{\begin{array}{rl} 
%      e^{\lambda (T-\tau)} w_{i_1,\ldots,i_k,i_y,i_u}(T-\tau,\tau_k,\ldots,\tau_1) & \mbox{ if }\\
%      & T-\tau \ge \tau_k \ge \cdots \ge \tau \\
%      0 &  \mbox{ otherwise}
%      \end{array}\right. 
%\end{align*}

%You can have as much text here as you want. The main body must be at most $8$ pages long.
%For the final version, one more page can be added.
%If you want, you can use an appendix like this one, even using the one-column format.
%%%%%%%%%%%%%%%%%%%%%%%%%%%%%%%%%%%%%%%%%%%%%%%%%%%%%%%%%%%%%%%%%%%%%%%%%%%%%%%
%%%%%%%%%%%%%%%%%%%%%%%%%%%%%%%%%%%%%%%%%%%%%%%%%%%%%%%%%%%%%%%%%%%%%%%%%%%%%%%

\section{Proof of Lemma \ref{lemma:main}}
\label{app_2}

The transition matrix $\Phi^{\uu}(t, \tau)$  belonging to the system
 (\ref{system:2}) for a fixed input function $\uu$ is the unique
  solution of the system $\dot{\Phi}^{\uu}(t, \tau) =
 A(\pu(t))\Phi^{\uu}(t, \tau)$ such that  $\Phi^\uu (\tau, \tau) = I$ for all
 $\tau \in [0, T]$.

\begin{lemma}(see Chapter 3.3.1.1 in \cite{toth2010modeling})
    \label{lemmatoth}
    For a system of form \ref{system:2} we have
    \begin{align*}
        \y_{\Sigma}(\uu, \p)(t) &= 
        %C(\pu(t))\Phi^{\uu}(t, 0)\psi(\uu) 
         \int_{0}^{T}C(\pu(\tau))\Phi^{\uu}(t, \tau)B(\pu(\tau))\uu(\tau)\,d\tau 
        %&+ \int_{0}^{T}C(\pu(\tau))\Phi^{\uu}(t, \tau)\mathbf{b}(\pu(\tau))\,d\tau 
         =: T_2
    \end{align*}
\end{lemma}

%Lemma \ref{lemmatoth}.  \\
%First of all, by Lemma \ref{helper1} it is sufficient to work with
%$\ysh(\psi(\uu), t)$. In order to apply Lemma \ref{helper2}, we wish to
%find an equal form of $\ysh(\psi(\uu), t)$ as a sum of inner products in
%some Hilbert space.
%Note, that $T_1 = 0$, because we restrict the initial state of system \ref{system:2} to 0. The term $T_3 = 0$ as well due to Remark \ref{remark:1}. In spite of these, 
We will show that 
%here exist Hilbert spaces $\mathcal{H}_1,\mathcal{H}_2$ and $\mathcal{H}_3$
   for which
    \begin{align*}
        \y_{\Sigma}(\uu, \p)(t) &= %\langle \w_{1}^{T, \Sigma},
        %\varphi_{1}^{T, \uu} %\rangle_{\mathcal{H}_1} 
        \langle \w_{2}^{T, \Sigma},
        \varphi_{2}^{T, \uu,\p}, \rangle_{\mathcal{H}_2} 
        %+\langle \w_{3}^{T, \Sigma},
        %\varphi_{3}^{T, \uu} %%\rangle_{\mathcal{H}_3},
    \end{align*}
    where $\w_{2}^{T, \Sigma}$ does not depend on $\uu$ or $\pi$, $\varphi_{i}^{T, \uu,\p}$
    does not depend on $\Sigma$ and they are respectively of
    $\mathcal{H}_2$ .
    %($1 \leq i \leq 3$). \\
Consider the system of the following form for all $1 \leq i_0 \leq n_p$.
\begin{align}
    \label{eq:system2}
    \begin{cases}
        \dot{\x}(t) = A_0 \x(t)
            + \sum\limits_{i=1}^{n_p}A_{i}\x(t) \pu_i(t) \\
        \x(0) = \psi(\uu) \\ %\Label{system} \\
        \y_{i_0}(\x(0), t) = C_{i_0}\x(t)
    \end{cases}
\end{align}
We know from \cite{isidori1985nonlinear} that for such systems
\begin{flalign*}
    &\y_{i_0}(\x(0), t) = \w_{i_0}^{t}
        +  \sum\limits_{k = 1}^{\infty} \sum \limits_{i_1,..,i_k = 1}^{n_p}
        \int_{0}^{t}\int_{0}^{\tau_k}\dots \int_{0}^{\tau_2}
        \w_{i_0,i_1 \dots i_k}^{t, \tau_k, \dots \tau_1}
        \pu_{i_k}(\tau_k)
        \dots \pu_{i_1}(\tau_1) \, d\tau_k \dots \, d \tau_1 
        \text{, where}& \\
        &\w_{i_0,\dots,i_k}^{t,\tau_k,\dots,\tau_1} = 
        \w_{i_0,\dots,i_k} (t, \tau_k,\dots,\tau_1)
        \in L_{2}([0, T]^{k + 1}, \mathbb{R}^{n})& \\
        &\w_i^{t} = \tilde{\w}_i^{t}\psi(\uu),  \quad
        \tilde{\w}_i^{t} = C_i \e^{A_0 t}, \quad
        \w_{i_0,\dots i_k}^{t,\tau_k,\dots,\tau_1} =
        \tilde{\w}_{i_0,\dots i_k}^{t,\tau_k,\dots,\tau_1}\psi(\uu) &\\
        & \tilde{\w}_{i_0,\dots i_k}^{t,\tau_k,\dots,\tau_1} =
        C_{i_0}\e^{A_{0}(t - \tau_k)}
        \Big(\prod\limits_{j = k}^{1}A_{i_j}\e^{A_0(\tau_{j} - \tau_{j -
        1})}\Big)\e^{A_0 \tau_1}& \\
        &\text{if } t \geq \tau_k \geq \dots \geq \tau_1 \text{ and 0 otherwise.}&
\end{flalign*}

\subsection{$T_2$ as scalar product}

In this section let $\mathbb{Q}$ denote the set of integers between $0$ and $n_p$, and $\mathbb{Q}^{\times}$ the set of all multiindices over $\mathbb{Q}$. Let us define the Hilbert space $\mathcal{H}_2$ equipped with an inner product
as follows. Let $\mathcal{H}_2$ consist of functions of the form
$f: \mathbb{Q}^{\times} \setminus\mathbb{Q} \rightarrow 
        \cup \bigcup_{k = 1}^{\infty}L_{2}([0, T]^k,
        \mathbb{R}^{n})$
such that 
%$f(q) \in \mathbb{R}^n$
%if $q \in \mathbb{Q}$ and 
for any $q,r \in \mathbb{Q}$
$f(qr) \in L_{2}([0, T], \mathbb{R}^n)$,
for any
$r,q_0,\ldots,q_k \in \mathbb{Q}$,
$k > 0$, 
$f(rq_0\cdots q_k) \in  L_{2}([0, T]^{k+1}, \mathbb{R}^n)$, 
and if $q_i=0$ for some $i=1,\ldots,k$, then
$f(rq_0\cdots q_k)=0$.
%if $w \in \mathbb{Q}^{\times} \setminus \mathbb{Q}$.%
%In case of $\mathcal{H}_1$ let $n = n_x$.
For such function $f$ 
\begin{align*}
        f \in \mathcal{H}_1 \Leftrightarrow 
        \sum\limits_{q \in \mathbb{Q}} \norm{f(q)}_{\mathbb{R}^n}^2
        + \sum\limits_{\substack{q \in \mathbb{Q}^{\times} \\  |w| > 1}}
        \norm{f(w)}_{L_{2}}%([0, T]^{|w| - 1}, \mathbb{R}^n)}^2
        < +\infty
\end{align*}
The inner product on $\mathcal{H}_1$ is defined as
\begin{align*}
    \langle f, g \rangle =
    \sum\limits_{q \in \mathbb{Q}} \langle f(q), g(q) \rangle_{\mathbb{R}^n}
    +\sum\limits_{w \in \mathbb{Q}^{\times} \setminus \mathbb{Q}}
    \langle f(w), g(w) \rangle_{L_{2}}%([0, T]^, \mathbb{R}^n)}
\end{align*}

Analogously to the previous case we will show that
$T_2$ has the an inner product formalization over a suitable choice of a
Hilbert space. The definition of $\mathcal{H}_2$ is
similar to the
definition
of $\mathcal{H}_1$ with the exception of $n = n_{in}$. The inner product in
this case is 
\begin{align*}
    \langle f, g \rangle =
    \sum\limits_{q, r \in \mathbb{Q}^2} \langle f(q, r), g(q, r)
    \rangle_{L_2([0, T], \mathbb{R}^n)}
    +
        \sum\limits_{w \in \mathbb{Q}^{\times} \setminus
        (\mathbb{Q} \cup \mathbb{Q}^2)}
    \langle f(w), g(w) \rangle_{L_{2}([0, T]^{|w| + 1}, \mathbb{R}^n)}
\end{align*}
%Let us fix a positive real constant $\lambda$.
Let us use the notation of Lemma \ref{ass4:lemma}
Let
    $\w_{2}^{T, \Sigma} (q, r)(\tau) = w_{q,r}(T-\tau)e^{\lambda (T-\tau)}$
%(as a function of $\tau$) for 
and for 
$r,q_0 \in \mathbb{Q}$ and 
any $q_1,\ldots,q_k \in \mathbb{Q}$, let
$\w_{2}^{T, \sigma} (r, q_0, \dots q_k) \in L_{2}([0, T]^{k+1}, \mathbb{R}^n)$
for $(r,q_0,\dots q_k) \in \mathbb{Q}^{\times}$ such that 
$\w_{2}^{T, \Sigma} (r,q_0, \dots q_k)(\tau, \tau_1,\dots,\tau_k) =$ 
$=w_{q_1,\dots q_k,q_0,r}(T - \tau, \tau_k, \dots \tau_1)e^{\lambda(T-\tau)}$ if
$T-\tau \ge \tau_k \ge \ldots \ge \tau_1$ and 
$q_1, \ldots, q_k$ are all difference from zero, and 
$\w_{2}^{T, \Sigma} (r,q_0, \dots q_k)(\tau, \tau_1,\dots,\tau_k) =0$ otherwise.
%\tilde{\w}_{q_0,\dots %q_k}^{T-\tau, \tau_k, \dots \tau_1} B_r \e^{\lambda (T
%- \tau)} $. 
Let $\varphi_{2}^{T,\tau, \uu, \p} (q, r) = p_{q}^{\uu}(T)
p_{r}^{\uu}(\tau)
\uu(\tau) \e^{-\lambda (T - \tau)}$ (as a function of $\tau$)
and 
$\varphi_{2}^{T, \tau,\uu, \p} (r,q_0, \dots q_k) \in L_{2}([0, T]^{k+1},
    \mathbb{R}^n)$ 
for $(r,q_0,\dots q_k) \in \mathbb{Q}^{\times}$ such that 
$\varphi_{2 | r,q_0, \dots q_k}^{T,\tau, \uu, \p | \tau_1,\dots,\tau_k} =
    \varphi_{2}^{T, \uu, \p} (r,q_0, \dots q_k)(\tau, \tau_k,\dots,\tau_1)
    =p_{q_k}^{\uu}(\tau_k + \tau) \dots p_{q_1}^{\uu}(\tau_1 + \tau) p_{q_0}^{\uu}(T)
        p^{\uu}_{r}(\tau) \uu(\tau) \e^{-\lambda (T - \tau)}$
if $T > \tau_k \geq \dots \geq \tau_1 \geq \tau$
and 0 otherwise.
Let
$y_{i_0, j_0}^{\tau}(t)$ 
be the output of the system \ref{eq:system2}
starting from the initial value $x(0) = 0$. Applying the
Volterra series expansion and taking the linear combination we formally have $T_2 %$C(\pu(t))\fiu(t,\tau)B(\pu(\tau))\uu(\tau)
    = \langle \w_{2}^{T, \Sigma},
\varphi_{2}^{T, \uu, \p} \rangle_{\mathcal{H}_2}$. We will show that they have
finite $\mathcal{H}_2$ norm.
\begin{lemma}
    $\varphi_{2}^{T, \uu, \p} \in \mathcal{H}_2$.
\end{lemma}
\begin{proof}
    Using the notations $\tau_{k+1} = T - \tau$,
   % \begin{align*} 
   %     J_{r,q_0,\dots,q_k}^{T,\tau,\tau_{k+1},\dots,\tau_1}
   %     = \int_{0}^{T} \int_{0}^{\tau_{k+1}}\dots \int_{0}^{\tau_2}
   % \norm{\varphi_{2 | r,q_0, \dots q_k}^{T, \uu | \tau_1,\dots,\tau_k}}^2
   %     \prod\limits_{i=1}^{k+1} \,d \tau_i 
   %     \,d \tau 
   % \end{align*}
    and the definition of $\mathcal{H}_2$ we have
    \begin{align*}
    \norm{\varphi_{2}^{T, \uu, \p}}_{\mathcal{H}_2}^2 &=
    \sum\limits_{r,q = 0}^{n_p}
    \int_{0}^{T}p_{q}^{\uu}(T)(p_{r}^{\uu})^2(\tau)\norm{\uu(\tau)}_2^2
    \e^{-\lambda (T - \lambda)}\,d\tau \\
    &+  \sum\limits_{k = 1}^{\infty} \sum_{r,q_0=0}^{n_p} \sum\limits_{q_1 \dots,q_k = 1}^{n_p}
      \int_{0}^{T} \int_{0}^{\tau_{k+1}}\dots \int_{0}^{\tau_2}
    \norm{\varphi_{2 | r,q_0, \dots q_k}^{T, \uu, \p | \tau_1,\dots,\tau_k}}^2
        \prod\limits_{i=1}^{k+1} \,d \tau_i 
        \,d \tau \\
    & \leq (n_p + 1)^2 \norm{\uu(\tau)}_{L_2}^2 
        + \sum\limits_{k = 1}^{\infty}\norm{\uu(\tau)}_{L_2}^2 
        \sum\limits_{r,q_0 \dots q_k = 0}^{n_p}\int_{0}^{T}
        \e^{-\lambda (T - \tau)}
        \int_{0}^{\tau_{k+1}}\dots\int_{\tau}^{\tau_2} \,d \tau_{1} \dots \,d
        \tau_{k+1} \,d \tau \\
        & \leq \norm{\uu(\tau)}_{L_2}^2 (n_p + 1)^2 
        \frac{\e^{(n_p + 1 - \lambda)T} - 1}{\lambda - n_p - 1}
        =: K_{\varphi_2}^2 (T)
        \leq \norm{\uu(\tau)}_{L_2}^2 (n_p + 1)^2 =: K_{\varphi_2}^2  < +\infty 
    \end{align*}
    for a suitable choice of $\lambda$
    as long as $\norm{\uu(\tau)}_{L_2}^2$ is bounded.
\end{proof}
\begin{lemma}
    \label{lemma:w2}
    $\w_{2}^{T, \Sigma} \in \mathcal{H}_2$.
\end{lemma}
\begin{proof}
    Using the notation $\tau_{k+1} = T - \tau$
    and the definition of $\mathcal{H}_2$ and Lemma \ref{ass4:lemma}we have
    \begin{align*}
    \norm{\w_{2}^{T, \Sigma}}_{\mathcal{H}_2}^2 
    &= \sum\limits_{r,q = 0}^{n_p}
    %\int_{0}^{T} \norm{C_q %\e^{A_0(T - \tau)}B_r}_2^2
    %\e^{\lambda (T - \lambda)}\,
    \int_0^{T} \| w_{r,q}(T-\tau)\|_2^2 e^{\lambda (T-\tau) } d\tau 
    \\
    & + \sum\limits_{k = 1}^{\infty} \sum\limits_{ q_1 \dots,q_k = 1}^{n_p}
    \int_{0}^{T} \int_{0}^{\tau_{k+1}}\dots \int_{0}^{\tau_2}
    \| w_{q_1,\ldots,q_k}(\tau_{k+1},\tau_k,\ldots,\tau_1)e^{\lambda \tau_{k+1}}
    d\tau_{k+1}\cdots \tau_{1}
    \le \\
    & \int_0^{\infty} \| w_{r,q}(T-\tau)\|_2^2 e^{\lambda (T-\tau) } d\tau 
    \\
    & + \sum\limits_{k = 1}^{\infty} \sum\limits_{ q_1 \dots,q_k = 1}^{n_p}
    \int_{0}^{\infty} \int_{0}^{\tau_{k+1}}\dots \int_{0}^{\tau_2}
    \| w_{q_1,\ldots,q_k}(\tau_{k+1},\tau_k,\ldots,\tau_1)e^{\lambda \tau_{k+1}}
    d\tau_{k+1}\cdots \tau_{1}
    \le c_1 \\
    \end{align*}
%        \prod\limits_{i=1}^{k+1} \,d \tau_i
 %       \,d \tau \\
 %       & \leq (n_p + 1)^2 K^2_C K^2_B \Bigg(
 %       \frac{1 - \e^{(\lambda - \gamma) T}}{\lambda}  
 %       + \sum\limits_{k = 1}^{\infty}
 %       \int_{0}^{T} \e^{(\lambda - \gamma)(T - \tau)}
        %\sum\limits_{r,q_0 \dots q_k = 0}^{n_p}\int_{0}^{T}
        %\int_{0}^{\tau_{k+1}}\dots\int_{\tau}^{\tau_2} \,d \tau_{1} \dots \,d
        %\tau_{k+1}
        %\, d \tau
        %\Bigg) \\
        %&\leq  (n_p + 1)^2 K^2_C K^2_B
        %\Big(\frac{1 - \e^{(\lambda - \gamma) T}}{\lambda} 
        %+ \frac{1 - \e^{(\lambda + n_p + 1 - \gamma)T}}{\gamma - \lambda - n_p - 1} \Big) 
         %=: K_{\w_2}^2(T)  \\
         %&< (n_p + 1)^2 K^2_C %K^2_B
        %\Big(\frac{1}{\lambda} 
        %+ \frac{1}{\gamma - %\lambda - n_p - 1} \Big)
        %=: K_{\w_2}^2 < +\infty 
    %\end{align*}
    %for a suitable choice of $\lambda$.
\end{proof}

%\subsection{Case of $T_3$}
%
%This is a special case of the previous one with 
%$\uu \equiv 1$ and $B_r = b_r$. As a result, $K_{\varphi_3}(T) \leq
%K_{\varphi_2}(T)$ and $K_{\w_3}(T) \leq \frac{K_b}{K_B}K_{\w_2}(T) $.

\section{Proof of Lemma \ref{lemma:sigma}}
\label{app_3}

Follows from the proof of Lemma \ref{lemma:w2}.

\section{Proof of Theorem \ref{theo:vc}}
\label{app_4}

 From \eqref{vc:eq1} it follows that it is enough to show that
 \begin{align*}
        \mathbb{P}_{S}\Bigg(\forall \Sigma \in \mathfrak{L}, \phi \in \Phi:
        \mathcal{L}(\Sigma,\phi) - \mathcal{L}_{N}(\Sigma,\phi)
         \leq \bar{R}(T, N, \delta)\Bigg) \geq 1 - \delta.
\end{align*}

To this end, notice that according to the classical argument (e.g. see the proof of Lemma 26.2 in \cite{shalev2014understanding})

   \begin{align*}
     G_{\mathfrak{L},\Phi}&:= E_{S} [ \sup_{\Sigma \in \mathfrak{L},\phi \in \Phi}
     \mathcal{L}(\Sigma,\phi) - \mathcal{L}_{N}(\Sigma,\phi)] \\
     &\leq \frac{1}{N} E_{S,S^{'}}[
     E_{\sigma_i \in \{+1,-1\}} [ 
     \sup_{\Sigma \in \mathfrak{L},\phi \in \Phi}
     \sum_{i=1}^{N} \sigma_i \left(\ell(\y_i,y_{\Sigma}(\uu_i,\phi(\uu_i),T)- \right. \left. \ell(\y_i^{'},y_{\Sigma}(\uu_i^{'},\phi(\uu_i^{'}),T) \right)]
   \end{align*}  

where $\sigma_{i}$, $i=1,\ldots,N$ are i.i.d uniformly distributed and taking values in $\{-1,1\}$, 
$S=\{\y_i,\uu(i)\}_{i=1}^{N}$
$S=\{\y^{'}_i,\uu^{'}(i)\}_{i=1}^{N}$ are independent
i.i.d. samples from $\mathcal{D}=(\mathcal{U},\mathcal{Y})$
Define
\begin{align*}
 &Z_{\phi}=\frac{1}{N} \sup_{\Sigma \in \mathfrak{L}} \sum_{i=1}^{N} \sigma_i  (\ell(y_i,y_{\Sigma}(\uu_i,\phi(\uu_i),T)-   \ell(y_i,y_{\Sigma}(\uu_i^{'},\phi(\uu_i^{'}),T) )  \\
&\mu_{\phi}= E_{\sigma_i \in \{+1,-1\}} [Z_{\phi}]
\end{align*} 

If $\Phi$ is a finite set, then

 \begin{align*}
    G_{\mathfrak{L},\Phi} 
      \le E_{S,S^{'}} E_{\sigma_i \in \{+1,-1\}} [\sup_{\phi \in \Phi} Z_{\phi}] \leq \mathrm{card}(\Phi) E_{S,S^{'}}[\sup_{\phi} \mu_{\phi}]
 \end{align*}

If $\Phi$ is not necessarily finite, then
    \begin{align*}
    G_{\mathfrak{L},\Phi} 
      \le E_{S,S^{'}} E_{\sigma_i \in \{+1,-1\}} [\sup_{\phi \in \Phi} Z_{\phi}] \leq E_{S,S^{'}} 
      E_{\sigma_i \in \{+1,-1\}} [\sup_{\phi \in \Phi} (Z_{\phi}-\mu_{\phi})] + E_{S,S^{'}} [\sup_{\phi \in \Phi} \mu_{\phi}]
   \end{align*}  
Let $H_{\phi}=Z_{\phi}-\mu_{\phi}$. It then follows
that $E_{\sigma_i \in \{+1,-1\}} [Z_{\phi}-\mu_{\phi})]$ and
$|H_{\phi}|=|Z_{\phi} -\mu_{\phi}| \le B$. Then by Hoeffding's lemma
$P(|H_{\phi}\| > \rho) < 2e^{-2 \frac{\rho^2}{4B^2}}$ and hence
by the union bound
$P(\sup_{\phi \in \Phi} |H_{\phi}\| > \rho) \le 2 \tau(2N) e^{-2\frac{\rho^2}{4B^2}}$, where 

\begin{align*}
\tau(2N)=\sup_{ \{\uu_i\}_{i=1}^{2N} \in   L_2([0,T],\mathbb{R}^{n_u})}  
| \{ (\uu_i,\phi(\uu_i)) \mid \phi \in \Phi\}\| \leq (2eN/d_T)^{d_T}
\end{align*}    

Then by Lemma A.4 in \cite{shalev2014understanding}

\begin{align*}
& E_{\sigma_i \in \{+1,-1\}} [\sup_{\phi \in \Phi} (Z_{\phi}-\mu_{\phi})] \leq \sqrt{2}B(T)(2+\sqrt{\log(2eN/d_T)d_T})
\end{align*}

As to $\sup_{\phi} \mu_{\phi}$, from standard theory of Rademacher
complexity and Lemma \ref{main} it follows that 
\[ \sup_{\phi} \mu_{\phi} \le     \frac{2K_{\ell} c_1(n_p+1)L_{\uu}}{\sqrt{N}} 
\]
That is, if $\Phi$ is finite, then
\[
G_{\mathfrak{L},\Phi} \le 2     \frac{\mathrm{card}(\Phi) 2K_{\ell} c_1(n_p+1)L_{\uu}}{\sqrt{N}} 
\]
and otherwise, 
\[
G_{\mathfrak{L},\Phi} \le \sqrt{2}B(2+\sqrt{\log(2eN/d_T)d_T}) +     \frac{2K_{\ell} c_1(n_p+1)L_{\uu}}{\sqrt{N}} 
\]
Applying McDiarmid inequality to
$\sup_{\phi \in \Phi} Z_{\phi}$
the statement of the theorem follows.

\end{document}